\pdfoutput=1
\documentclass[final]{hld2025} 

\usepackage[T1]{fontenc}
\usepackage{lmodern}
\usepackage{bm}
\usepackage{booktabs}
\usepackage{multirow}
\usepackage{float}
\usepackage{wrapfig}
\hypersetup{hypertexnames=false}

\newtheorem{observation}{Observation}[section]

\makeatletter
\renewcommand*{\@titlefoot}{}
\renewcommand\ps@jmlrtps{%
  \let\@mkboth\@gobbletwo
  \def\@oddhead{}%
  \let\@evenhead\@oddhead
  \def\@oddfoot{\hfill \small\rmfamily \thepage \hfill}%
  \let\@evenfoot\@oddfoot
}
\makeatother

\title[Grouping Attention Heads for Muon]{When and Why Grouping Attention Heads Accelerates Muon Optimization}

\hldauthor{%
\Name{Hongtao Zhang}\thanks{Equal contribution} \Email{zhanghongtao24s@ict.ac.cn}\\
\addr {State Key Laboratory of AI Safety, Institute of Computing Technology, Chinese Academy of Sciences; School of Advanced Interdisciplinary Sciences, University of Chinese Academy of Sciences} 
\AND
\Name{Wenjie Zhou}\footnotemark[1] \Email{zj4323005@gmail.com}\\
\addr {State Key Laboratory of AI Safety, Institute of Computing Technology, Chinese Academy of Sciences; University of Chinese Academy of Sciences}
\AND
\Name{Wei Chen}\thanks{Corresponding author} \Email{chenwei2022@ict.ac.cn}\\
\addr {State Key Laboratory of AI Safety, Institute of Computing Technology, Chinese Academy of Sciences; University of Chinese Academy of Sciences}
\AND
\Name{Xueqi Cheng} \Email{cxq@ict.ac.cn}\\
\addr {State Key Laboratory of AI Safety, Institute of Computing Technology, Chinese Academy of Sciences; University of Chinese Academy of Sciences}
}

\begin{document}

\maketitle

\begin{abstract}%
Muon orthogonalizes matrix updates, but multi-head attention naturally operates
at the level of heads. This granularity mismatch raises the question of whether
Muon should be applied to the full attention projection, to individual heads, or
to intermediate head groups. We study this question through a one-step descent
comparison between full-matrix Muon and group-wise Muon.
Our analysis reveals a trade-off between the \textbf{group-wise whitening gain} from group-wise updates and the \textbf{grouping-induced norm cost}, an additional update-norm cost caused by replacing full-matrix whitening with group-wise whitening.
Motivated by this trade-off, we propose \textbf{Group Muon}, which treats head group size and grouping rule as optimizer hyperparameters. On GPT-2 Small trained on FineWeb,
appropriate grouping improves validation loss over both full-QKV Muon and
fully head-wise MuonSplit.
\end{abstract}


\section{Introduction}

Muon~\citep{jordan2024muon} orthogonalizes the momentum update of each 2-D parameter
matrix, so its natural unit is a matrix. Multi-head attention, however, later
splits the query, key, and value projections into heads that compute attention
independently. Applying Muon to the full projection couples all heads through a
single polar factor, whereas splitting the projection gives each head, or group
of heads, an independent orthogonalized update. We ask:

\vspace{-0.75em}
\begin{center}
\emph{Q: What is the right granularity for applying Muon to multi-head attention?}
\end{center}
\vspace{-0.75em}

This question is already practically relevant. GLM-5~\citep{glm5team2026glm5vibecodingagentic} uses Muon Split, applying
orthogonalization to per-head blocks of attention projections for MLA-based
attention. A record on the modded-nanoGPT speedrun~\citep{Speedrun-PR253}  also recently improved its
record by orthogonalizing Q and K in pairs of heads rather than across the full
multi-head matrix. These examples suggest that Muon's
matrix scope is an optimizer design choice, not a mere implementation detail.
Appendix~\ref{app:related-work} provides a fuller related-work discussion.

Yet finer splitting is not always better. As shown in Figure~\ref{fig1},
head-wise MuonSplit reduces validation loss faster early in training, but is
later overtaken by full-matrix Muon. This suggests a stage-dependent trade-off:
splitting can improve descent at one stage while degrading optimization at
another.

\begin{figure}[htbp]
\centering
\includegraphics[width=1\textwidth]{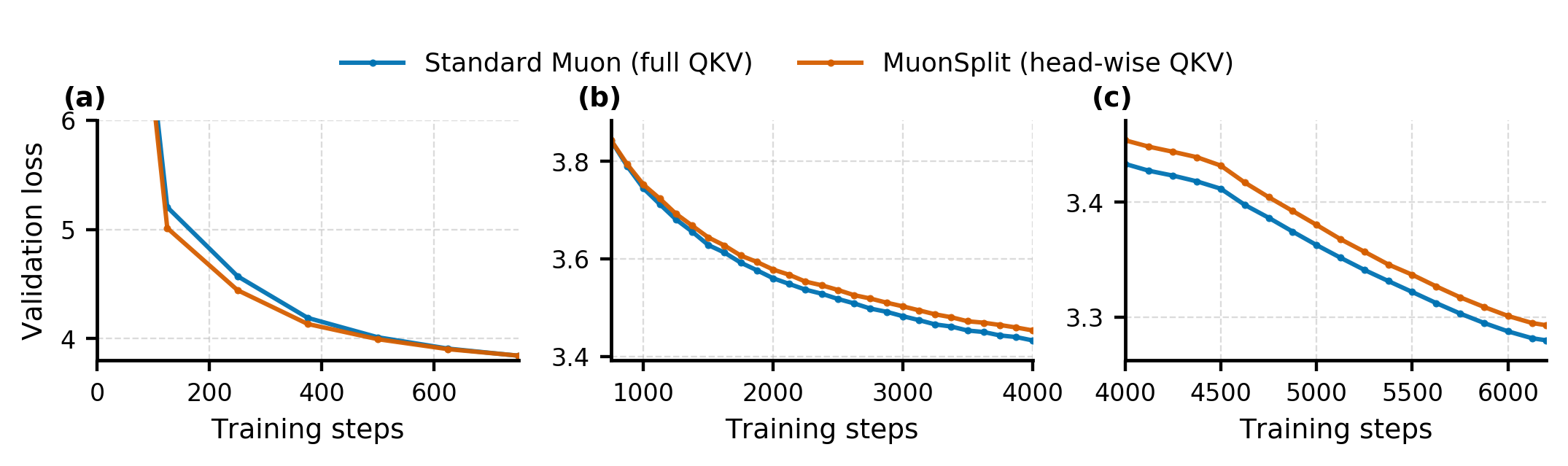}
\vspace{-0.75em}
\caption{Validation loss of full-matrix Muon and head-wise MuonSplit in the
GPT-2 Small speedrun setting. MuonSplit descends faster early, whereas
full-matrix Muon performs better later.}
\label{fig1}
\vspace{-1em}
\end{figure}
To answer this granularity question, we make the following contributions:

\begin{itemize}
    \item We provide a one-step descent analysis comparing full-matrix Muon and
    group-wise Muon. The resulting criterion identifies a trade-off between the
    \textbf{group-wise whitening gain}, which increases the first-order descent
    term from \(\|G\|_*\) to \(\sum_i \|G_i\|_*\), and the
    \textbf{grouping-induced norm cost}, a second-order cost induced by the
    larger norm of group-wise whitened updates. This criterion explains the
    early/late reversal in Figure~\ref{fig1}: 
   finer splitting can help when the group-wise whitening gain outweighs the grouping-induced norm cost, but may underperform when this gain is insufficient to compensate for the norm cost introduced by grouping.

    \item We propose \textbf{Group Muon}, a practical head-grouped Muon variant
    that treats head group size and grouping rule as explicit optimizer
    hyperparameters. On GPT-2 Small trained on FineWeb, appropriate intermediate
    grouping improves validation loss over both Full-QKV Muon and fully
    head-wise MuonSplit.
\end{itemize}


\section{Preliminaries}
\paragraph{Block structure.}
Partition $ W \in \mathbb{R}^{m \times d}$ into row groups
$W=[W_1^\top\cdots W_M^\top]^\top$,
where $W_i\in \mathbb{R}^{m_i\times d}$ and $\sum_{i=1}^M m_i=m$.
For objective $\mathcal{L}(W)$, set $G:=\nabla_W \mathcal{L}(W)\in \mathbb{R}^{m\times d}$ and
$G=[G_1^\top\cdots G_M^\top]^\top$, where $G_i:=\nabla_{W_i}\mathcal{L}(W)\in \mathbb{R}^{m_i\times d}$.
Set $r:=\mathrm{rank}(G) $ and $ r_i:=\mathrm{rank}(G_i)$.

\paragraph{Polar factor.}
Idealized Muon uses the gradient polar factor. For nonzero $X=U_X\Sigma_X V_X^{\top}$ with compact SVD, set $P(X):=U_XV_X^{\top}$ and $P(0)=0$. Then
$\label{polarFactor_identities}
    \langle{X}, {P(X)}\rangle=\|X\|_* ,
    \|P(X)\|_F^2=\mathrm{rank}(X),
$
where $\| \cdot \|_*$ denotes the nuclear norm.

\paragraph{Update rules.}
\textit{Full Matrix Muon} polar-whitens $G$ globally:
\begin{equation} \label{muon update}
    O_{\mathrm{all}}=P(G),
    \qquad
    W^+_{\mathrm{all}}=W-\eta O_{\mathrm{all}}.
\end{equation}
\textit{Group Muon} polar-whitens each $G_i$ separately:
\begin{equation} \label{group muon update}
    O_{\mathrm{grp}}=
    \begin{bmatrix}
        P(G_1)^\top&
        P(G_2)^\top&
        \cdots&
        P(G_M)^\top
    \end{bmatrix}^\top,
    \qquad
    W^+_{\mathrm{grp}}=W-\eta O_{\mathrm{grp}}.
\end{equation}

\paragraph{Experiment Setup}
\textbf{GPT-2 on FineWeb}: We train GPT-2 Small (124M) model on dataset using the highly optimized \texttt{nano-gpt} benchmark training recipe \cite{radford2019language} \footnote{\url{https://github.com/KellerJordan/modded-nanogpt}}. Details are provided in Appendix~\ref{app:group_head_muon_experimental_setting}.

\section{When and Why Grouping Helps: Gain versus Norm Cost}
\label{sec:when and why theoretical analysis}
We first compare full matrix Muon and Group Muon through a one-step smoothness bound. 
Assume that $\mathcal{L}$ is $\beta$-smooth with respect to $W$, i.e.,
\begin{equation}\label{beta_smooth_ass}
    \mathcal{L}(W+\Delta)
    \le
    \mathcal{L}(W)
    +\langle \nabla\mathcal{L}(W),\Delta\rangle
    +\frac{\beta}{2}\|\Delta\|_F^2 .
\end{equation}

\begin{proposition}[One-step descent comparison]
\label{prop:one-step-comparison}
Let $G=\nabla_W\mathcal{L}(W)$ and let $G_i$ be the $i$-th group block of $G$. 
Denote $r:=\mathrm{rank}(G), r_i:=\mathrm{rank}(G_i).$
For the full matrix Muon and Group Muon the corresponding guaranteed one-step decrease lower bounds are
\begin{equation*}
    D_{\mathrm{all}}
    =
    \eta\|G\|_*
    -\frac{\beta\eta^2}{2}r,
    \qquad
    D_{\mathrm{grp}}
    =
    \eta\sum_{i=1}^M\|G_i\|_*
    -\frac{\beta\eta^2}{2}\sum_{i=1}^M r_i .
\end{equation*}
Consequently, Group Muon has a larger guaranteed descent lower bound than full matrix Muon whenever
\begin{equation}
    \boxed{
    \sum_{i=1}^M \|G_i\|_*-\|G\|_*
    >
    \frac{\beta\eta}{2}
    \left(
        \sum_{i=1}^M r_i-r
    \right).
    }
    \label{ieq:main-criterion}
\end{equation}
\end{proposition}
A detailed proof is provided in Appendix~\ref{app:proof-one-step-comparison}.

\paragraph{Gain--cost trade-off.}
In \eqref{ieq:main-criterion}, the left-hand side is the \textbf{group-wise whitening gain}, i.e., the first-order benefit of group-wise whitening. 
The right-hand side is the \textbf{grouping-induced norm cost}, which measures the additional update-norm cost caused by replacing full-matrix whitening with group-wise whitening.


\textbf{Takeaway.} This criterion gives a simple answer to when grouping helps: Group Muon is favored when the group-wise whitening gain outweighs the grouping-induced norm cost. 
Thus, the relative size of these two terms determines whether grouping improves over full matrix Muon. 
We first consider the small-cost regime.

\paragraph{Small-cost regime: near-full-rank gradients.}
\begin{observation}[Near-full-rank gradients]
\label{obs:near-full-rank}
Empirically, during the main training regime, the full gradient and the group gradients are close to full row rank. 
For a row-wise partition with $m\le d$, this corresponds to
$
    r\approx m, r_i\approx m_i, \sum_{i=1}^M r_i\approx r .
$
\end{observation}

We provide supporting empirical evidence for this observation in Appendix~\ref{appendix:fullrank}.

\paragraph{Zero-cost case: full row rank.}
If $\mathrm{rank}(G)=m$ and $\mathrm{rank}(G_i)=m_i$ for all groups, then $\sum_i r_i-r=0$, so \eqref{ieq:main-criterion} reduces to $\sum_i\|G_i\|_*>\|G\|_*$. 
By nuclear-norm subadditivity, $\|G\|_*\le\sum_i\|G_i\|_*$, with equality for the row-wise partition iff the block row spaces are mutually orthogonal. Thus, ideal Muon zero-cost grouping is generically favored.


\paragraph{Practical cost: effective rank and noise.}
For ideal Muon with the exact polar factor, the grouping-induced norm cost reduces to the rank-dependent term because
$
    \|P(G)\|_F^2=\mathrm{rank}(G).
$
However, practical Muon computes an approximate whitening direction, e.g., through a finite number of Newton--Schulz iterations, and the gradients contain stochastic noise.
Appendix~\ref{appendix_Grouping-induced-Norm-Cost} derives this non-ideal form and empirically shows that the Frobenius-norm gap, so the grouping-induced norm cost should not be assumed to vanish in practical Muon implementations.
In this case, the exact algebraic rank is not the most direct proxy for the second-order term in the descent bound.

\paragraph{Over-splitting risk: aligned low-rank gradients.}

In an aligned low-rank case with $k$ nonzero groups, the grouping-induced norm cost is
$
    \frac{\beta\eta}{2}(k-1),
$
which increases monotonically with $k$. 
This indicates that the grouping-induced norm cost should not be assumed to vanish in practical Muon implementations. We provide more details in Appendix~\ref{app:aligned-low-rank-example}.


Thus, Grouping can increase the group-wise whitening gain, but overly fine splitting can also increase the grouping-induced norm cost through larger whitened update norms.
Therefore, the key design question is not simply whether to split the QKV, but how coarsely the attention heads should be grouped.

\section{Group Muon Algorithm}
\label{sec:experiments}

The analysis above suggests that Group Muon's effectiveness depends on grouping structure: overly fine splitting may add grouping-induced norm cost, while intermediate groups can preserve group-wise whitening gain. We therefore introduce a practical Group Muon algorithm with tunable head group size and evaluate grouping rules. 

We form groups over attention heads using a head group size $g$ and a grouping rule $\mathcal{S}$. 
Here, $g$ denotes the number of heads per group, and $\mathcal{S}$ specifies how heads are assigned. 
We consider adjacent $(\mathrm{adj})$, interval $(\mathrm{int})$, and random $(\mathrm{rand})$ grouping, corresponding to consecutive heads, maximally separated heads, and randomly resampled size-$g$ partitions, respectively. 
More details are provided in Appendix~\ref{app:group_head_muon_experimental_setting}.

Based on these definitions, we propose the Group Muon update described in Algorithm~\ref{alg:group-muon}.

We apply Group Muon to attention projection matrices in three settings: QK groups the query and key projections, V groups the value projection, and Fixed-QK+V fixes QK to its best observed grouping while varying V.

\begin{wrapfigure}[17]{r}{0.52\textwidth}
\begin{algorithm2e}[H]
\scriptsize
\label{alg:group-muon}
 \caption{GroupMuon}
 \SetAlgoLined
  \KwData{parameter matrix $W$, gradient $G$, momentum buffer $M$, learning rate $\eta$, momentum coefficient $\mu$, number of heads $H$, head group size $g$, grouping rule $\mathcal{S}$}
  \KwResult{updated parameter matrix $W$}
  split $G$ and $M$ into head-wise blocks $\{G_h\}_{h=1}^{H}$ and $\{M_h\}_{h=1}^{H}$\;
  construct head groups $\mathcal{P}=\{P_1,\ldots,P_{H/g}\}$ according to $\mathcal{S}$\;
  \ForEach{group $P \in \mathcal{P}$}{
    $\widetilde{G}_P \leftarrow \mathrm{Concat}(\{G_h: h \in P\})$\;
    $\widetilde{M}_P \leftarrow \mathrm{Concat}(\{M_h: h \in P\})$\;
    $\widetilde{M}_P \leftarrow \mu \widetilde{M}_P + \widetilde{G}_P$\;
    $\widetilde{O}_P \leftarrow \mathrm{NewtonSchulz}(\widetilde{M}_P)$\;
    split $\widetilde{O}_P$ back into head-wise updates $\{O_h: h \in P\}$\;
  }
  merge all head-wise updates into $O$\;
  $W \leftarrow W - \eta O$\;
  return $W$\;
\end{algorithm2e}
\vspace{-0.5em}
\end{wrapfigure}
As a representative case, we analyze the QK random setting to illustrate how the head group size affects the balance between whitening gain and norm cost.
In the early stage(Fig~\ref{fig2}(a)), the finest grouping $g=1$ achieves the lowest validation loss, indicating that fine-grained head-wise whitening can be beneficial at the beginning of training. 
This is consistent with Observation~\ref{obs:near-full-rank}: when the full gradient and group gradients are close to full row rank, the idealized grouping-induced norm cost in \eqref{ieq:main-criterion} is small, so the group-wise whitening gain can dominate the comparison.

However, this fine-splitting advantage does not persist throughout training. 
As training proceeds(Fig~\ref{fig2}(b)), the coarser QK random grouping with $g=6$ surpasses $g=1$ and gradually becomes the best-performing setting in both Figure~\ref{fig2}(c) and Table~\ref{tab:group_results}. 
This reversal is consistent with the gain--cost view: finer splitting may increase the group-wise whitening gain, but it also introduces a grouping-induced norm cost. 
When the additional gain no longer compensates for this cost, a coarser grouping can become preferable. 
This also helps explain why fully head-wise Muon does not necessarily outperform Full-QKV Muon, despite its larger first-order group-wise whitening term.

\begin{table}[H]
\centering
\vspace{-0.5em}
\caption{
Validation loss under different Group Muon settings.
In the Fixed-QK+V setting, QK is fixed to its best grouping configuration($g=6$, rand.), while the V grouping is varied.
}
\vspace{-0.75em}
\label{tab:group_results}
\footnotesize
\setlength{\tabcolsep}{4pt}
\renewcommand{\arraystretch}{0.88}
\begin{tabular}{llccccc}
\toprule
\multicolumn{2}{l}{Baselines} 
& \multicolumn{5}{c}{Validation loss} \\
\midrule
\multicolumn{2}{l}{Full-QKV Muon} 
& \multicolumn{5}{c}{3.2780} \\
\multicolumn{2}{l}{Head-wise QKV MuonSplit} 
& \multicolumn{5}{c}{3.2799} \\

\midrule
GroupMuon(\textbf{ours}) & Rule 
& $g=1$ & $g=2$ & $g=3$ & $g=4$ & $g=6$ \\
\midrule

\multirow{3}{*}{QK}
& adj. 
& 3.2787 & 3.2774 & 3.2770 & 3.2757 & 3.2743 \\
& int. 
& 3.2778 & 3.2776 & 3.2767 & 3.2745 & 3.2750 \\
& rand. 
& \textbf{3.2759} & \textbf{3.2758} & \textbf{3.2746} & \textbf{3.2731} & \textbf{3.2722} \\

\midrule

\multirow{2}{*}{V}
& adj. 
& 3.2806 & 3.2765 & 3.2780 & 3.2743 & 3.2781 \\
& rand. 
& 3.2806 & 3.2792 & 3.2769 & 3.2765 & 3.2769 \\

\midrule

\multirow{2}{*}{Fixed-QK+V}
& adj. 
& 3.2799 & 3.2770 & 3.2762 & 3.2761 & 3.2746 \\
& rand. 
& 3.2788 & 3.2782 & 3.2740 & 3.2744 & 3.2752 \\

\bottomrule
\end{tabular}
\vspace{-0.75em}
\end{table}


By contrast, $g=6$ provides a more stable grouping scale in the later stage, retaining more structure than Full-QKV Muon while avoiding excessive head-wise fragmentation. 
Accordingly, Q/K random grouping with $g=6$ achieves the best validation loss among the settings, whereas additional V grouping brings no consistent improvement. 
This suggests that moderately coarse Q/K grouping can be preferable for sustained training performance.

\begin{figure}[H]
\centering
\includegraphics[width=0.98\textwidth]{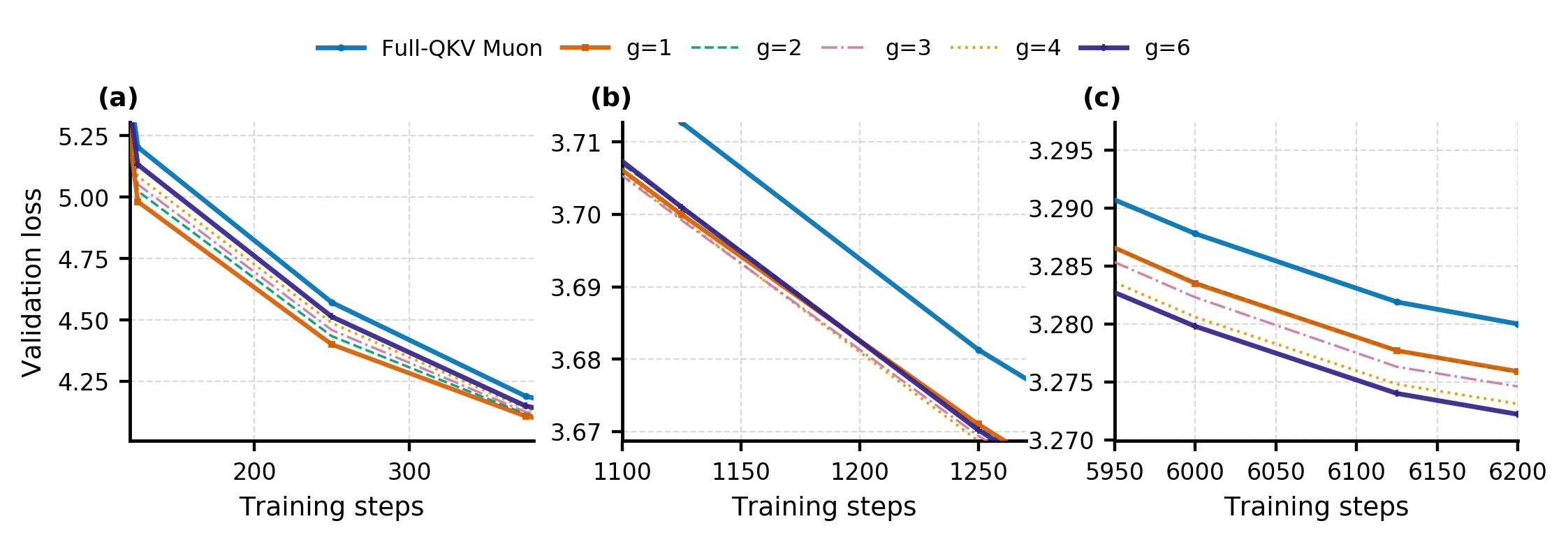}
\vspace{-1em}
\caption{
Validation loss of Full-QKV Muon and random Group Muon. 
In the early stage (a), smaller group sizes perform better, with $g=1$ achieving the lowest validation loss. 
In the middle and late stages, (b) and (c), larger group sizes become more favorable, with $g=6$ surpassing $g=1$ and performing best.
}
\label{fig2}
\vspace{-1em}
\end{figure}

\section{Conclusion}

Grouping attention heads can accelerate Muon when the group-wise whitening gain outweighs the grouping-induced norm cost. 
This gain--cost view motivates intermediate grouping scales, and our experiments find Q/K random grouping with $g=6$ to be the best among the tested settings.

\newpage

\bibliography{sample}
\newpage
\clearpage
\appendix

\section{Related Work}
\label{app:related-work}

\paragraph{Muon for large-scale language-model training.}
Muon applies matrix orthogonalization to optimizer updates and has rapidly become a practical alternative to AdamW for hidden layers in neural-network training~\citep{jordan2024muon}. 
Recent large-scale studies show that Muon can improve the compute--loss trade-off in LLM pretraining~\citep{liu2025muonscalablellmtraining,ai2025practicalefficiencymuonpretraining}, while complementary analyses investigate learning-rate transfer, spectral regularization, weight decay, and the effect of approximate Newton--Schulz orthogonalization~\citep{boreiko2025towards,shulgin2025idealanalyzinginexactmuon}. 
Our work is orthogonal to these scaling and approximation studies: we keep the Muon update family fixed and ask which attention-head granularity should define the matrix blocks being orthogonalized.

\paragraph{Block-wise and distributed orthogonalization.}
Several recent methods reduce the cost of orthogonalized updates by changing the matrix on which orthogonalization is performed. 
Dion replaces dense Newton--Schulz orthogonalization with distributed low-rank orthonormalized updates designed for sharded weights~\citep{ahn2025diondistributedorthonormalizedupdates}, and Dion2 further simplifies this idea by sampling rows or columns to shrink the orthogonalization problem~\citep{ahn2025dion2simplemethodshrink}. 
Block-wise Muon variants also split weight matrices into tiles to make tensor-parallel training more efficient~\citep{boreiko2025towards}. 
MuonBP applies block-wise orthogonalization on device-local matrix shards and periodically performs full orthogonalization to recover the stability of baseline Muon~\citep{khaled2025muonbpfastermuonblockperiodic}. 
The key distinction from MuonBP is that our blocks are \emph{attention-head groups}, not primarily device shards: MuonBP is motivated by distributed communication and throughput, whereas Group Muon treats the head group size and grouping rule as optimizer hyperparameters and analyzes the trade-off between group-wise whitening gain and grouping-induced norm cost.

\paragraph{Attention-head structure.}
Our focus on head grouping is motivated by the internal multi-head structure of attention. 
Recent work models multi-head attention as an interacting multi-player system, emphasizing that different heads can compete, coordinate, and induce cross-head externalities~\citep{chakrabarti2026multiheadattentionmultiplayergame}. 
This perspective is consistent with treating attention heads as meaningful optimizer units, but our goal is different: rather than regularizing head interactions, we study how the orthogonalization scope of Muon should be chosen over heads or groups of heads.

\section{Proof of Proposition~\ref{prop:one-step-comparison}}
\label{app:proof-one-step-comparison}

\begin{proof}
By $\beta$-smoothness\eqref{beta_smooth_ass}, for any update direction $O$ and stepsize $\eta>0$, taking $\Delta=-\eta O$ gives
\begin{equation}
    \mathcal{L}(W-\eta O)
    \le
    \mathcal{L}(W)
    -\eta \langle G,O\rangle
    +\frac{\beta\eta^2}{2}\|O\|_F^2 .
    \label{eq:appendix-smoothness}
\end{equation}

We first apply \eqref{eq:appendix-smoothness} to the full matrix Muon direction
\[
    O_{\mathrm{all}}=P(G).
\]
Using the polar-factor identities
\[
    \langle G,P(G)\rangle=\|G\|_*,
    \qquad
    \|P(G)\|_F^2=\mathrm{rank}(G)=r,
\]
we obtain
\begin{equation}
    \mathcal{L}(W^+_{\mathrm{all}})
    \le
    \mathcal{L}(W)
    -\eta\|G\|_*
    +\frac{\beta\eta^2}{2}r .
\end{equation}
Therefore the guaranteed one-step decrease lower bound for full matrix Muon is
\begin{equation}
    D_{\mathrm{all}}
    =
    \eta\|G\|_*
    -\frac{\beta\eta^2}{2}r .
\end{equation}

Next consider the Group Muon direction
\[
    O_{\mathrm{grp}}
    =
    \begin{bmatrix}
        P(G_1)^\top&
        P(G_2)^\top&
        \cdots&
        P(G_M)^\top
    \end{bmatrix}^\top .
\]
Since $G$ and $O_{\mathrm{grp}}$ share the same block structure, we have
\begin{equation}
    \langle G,O_{\mathrm{grp}}\rangle
    =
    \sum_{i=1}^M \langle G_i,P(G_i)\rangle
    =
    \sum_{i=1}^M \|G_i\|_* .
\end{equation}
Similarly, because the blocks are disjoint,
\begin{equation}
    \|O_{\mathrm{grp}}\|_F^2
    =
    \sum_{i=1}^M \|P(G_i)\|_F^2
    =
    \sum_{i=1}^M \mathrm{rank}(G_i)
    =
    \sum_{i=1}^M r_i .
\end{equation}
Substituting these two identities into \eqref{eq:appendix-smoothness} gives
\begin{equation}
    \mathcal{L}(W^+_{\mathrm{grp}})
    \le
    \mathcal{L}(W)
    -\eta\sum_{i=1}^M \|G_i\|_*
    +\frac{\beta\eta^2}{2}\sum_{i=1}^M r_i .
\end{equation}
Thus the guaranteed one-step decrease lower bound for Group Muon is
\begin{equation}
    D_{\mathrm{grp}}
    =
    \eta\sum_{i=1}^M\|G_i\|_*
    -\frac{\beta\eta^2}{2}\sum_{i=1}^M r_i .
\end{equation}

Group Muon has a larger guaranteed descent lower bound than full matrix Muon when
\[
    D_{\mathrm{grp}}>D_{\mathrm{all}}.
\]
Substituting the two expressions above, this condition becomes
\[
    \eta\sum_{i=1}^M\|G_i\|_*
    -\frac{\beta\eta^2}{2}\sum_{i=1}^M r_i
    >
    \eta\|G\|_*
    -\frac{\beta\eta^2}{2}r .
\]
Rearranging and dividing by $\eta>0$ yields
\[
    \sum_{i=1}^M \|G_i\|_*-\|G\|_*
    >
    \frac{\beta\eta}{2}
    \left(
        \sum_{i=1}^M r_i-r
    \right),
\]
which proves the comparison criterion.
\end{proof}

\section{Experimental Setting}

\label{app:group_head_muon_experimental_setting}

\paragraph{GPT-2 Small on FineWeb.}

\paragraph{Models and Architecture.}
We conduct all head-grouping experiments on a GPT-2 Small style decoder-only Transformer \citep{radford2019language, vaswani2017attention} trained from scratch on FineWeb. The model uses $n_{\text{layer}}=12$, $d_{\text{model}}=768$, and $n_{\text{head}}=12$, giving a head dimension of $d_{\text{head}}=64$. Following the modded NanoGPT-style implementation, the architecture uses tied input/output embeddings, RMSNorm\citep{zhang2019root}, Rotary Positional Embeddings (RoPE)\citep{su2024roformer}, bias-free linear layers, and a decoder-only causal attention stack.

\paragraph{Dataset and Tokenization.}
All experiments use the FineWeb10B binary shards \citep{penedo2024fineweb} with the GPT-2 BPE vocabulary of size 50,257. We use a context length of $L=1024$ tokens. Validation loss is evaluated every 125 optimization steps on a fixed validation budget of 10,485,760 tokens.

\paragraph{Shared Hyperparameters.}
All runs use 4-GPU distributed data parallel training with a global batch size of 512 sequences, corresponding to 524,288 tokens per optimization step. Each experiment is trained for 6,200 steps, or approximately 3.25B training tokens. We use a constant learning-rate phase followed by a linear warmdown over the final 1,800 steps, with no warmup.

\paragraph{Optimizer.}
We use a split optimization strategy. The language-model head and tied token embeddings are optimized with AdamW using learning rate $3.6\times 10^{-3}$, betas $(0.9,0.95)$, and zero weight decay. The Transformer blocks are optimized with Muon using learning rate $3.6\times 10^{-4}$ and momentum $0.95$. Unless otherwise specified, Muon orthogonalization uses the Newton-Schulz backend; grouped head updates use the batched Newton-Schulz backend. 

\paragraph{Optimizer.}
We use a split optimization strategy. 
The language-model head and tied token embeddings are optimized with AdamW using learning rate $3.6\times 10^{-3}$, betas $(0.9,0.95)$, and zero weight decay. 
The Transformer blocks are optimized with Muon using base learning rate $3.6\times 10^{-4}$ and momentum $0.95$. 
For head-wise and group-wise Muon variants, splitting a Q/K/V matrix changes the effective matrix shape on which the polar-factor update is applied. 
We therefore apply a shape-aware learning-rate transfer rule\footnote{We follow the Muon learning-rate transfer practice described in \url{https://kexue.fm/archives/11416}.}: the Muon learning rate is transferred from the original full-matrix shape to the corresponding grouped matrix shape, so that the update scale remains comparable across Full-QKV Muon, Head-wise QKV Muon, and Group Muon variants. 
Unless otherwise specified, Muon orthogonalization uses the Newton--Schulz backend, and grouped head updates use the batched Newton--Schulz backend.

\begin{table}[h]
\centering
\caption{Shared hyperparameters for the GPT-2 Small head-grouping experiments.}
\label{tab:group_head_muon_shared_hyperparams}
\resizebox{\textwidth}{!}{%
\begin{tabular}{lcccccccc}
\toprule
\textbf{Model} & \textbf{Steps} & \textbf{Tokens} & \textbf{Seq. Len.} & \textbf{Global Batch} & \textbf{Head LR} & \textbf{Body LR} & \textbf{WD} & \textbf{Schedule} \\
\midrule
GPT-2 Small & 6,200 & $\sim$3.25B & 1024 & 512 & $3.6\text{e-}3$ & $3.6\text{e-}4$ & 0.0 & Linear warmdown, 1,800 steps \\
\bottomrule
\end{tabular}%
}
\end{table}

\paragraph{Compared Methods.}
We organize the experiments according to the same categories used in Table~\ref{tab:group_results}. 
The two baselines are not parameterized by the group size $g$. 
\textbf{Full-QKV Muon} applies the standard Muon update to the packed QKV projection matrix. 
\textbf{Head-wise QKV Muon} first splits the packed QKV projection into Q, K, and V blocks, then further splits each block into individual attention heads and applies Muon independently to each head.

For grouped variants, we use the attention head as the basic unit. 
Since GPT-2 Small has $H=12$ attention heads, a group size $g\in\{1,2,3,4,6\}$ partitions the heads into $H/g$ groups. 
We consider three grouped settings:
\begin{itemize}
    \item \textbf{QK}: apply Group Muon only to the Q and K projections;
    \item \textbf{V}: apply Group Muon only to the V projection;
    \item \textbf{Fixed-QK+V}: fix Q and K to the best QK configuration, namely $g=6$ with random grouping, and sweep the grouping configuration of V.
\end{itemize}

\begin{table}[h]
\centering
\caption{
Experiment settings corresponding to Table~\ref{tab:group_results}. 
For grouped variants, $g$ denotes the number of heads per group. 
In random grouping, the head partition is resampled at every optimizer step.
}
\label{tab:group_head_muon_variants}
\resizebox{\textwidth}{!}{%
\begin{tabular}{llll}
\toprule
\textbf{Setting} & \textbf{Grouped target} & \textbf{Grouping rule} & \textbf{Sweep} \\
\midrule
Full-QKV Muon 
& QKV 
& -- 
& -- \\

Head-wise QKV Muon 
& QKV 
& per-head 
& -- \\

\midrule
QK 
& Q, K 
& adj. 
& $g\in\{1,2,3,4,6\}$ \\

QK 
& Q, K 
& int. 
& $g\in\{1,2,3,4,6\}$ \\

QK 
& Q, K 
& rand. 
& $g\in\{1,2,3,4,6\}$ \\

\midrule
V 
& V 
& adj. 
& $g_V\in\{1,2,3,4,6\}$ \\

V 
& V 
& rand. 
& $g_V\in\{1,2,3,4,6\}$ \\

\midrule
Fixed-QK+V 
& Q, K, V 
& adj. for V 
& $g_{QK}=6$ with rand., $g_V\in\{1,2,3,4,6\}$ \\

Fixed-QK+V 
& Q, K, V 
& rand. for V 
& $g_{QK}=6$ with rand., $g_V\in\{1,2,3,4,6\}$ \\
\bottomrule
\end{tabular}%
}
\end{table}

\paragraph{Grouping Rules.}
We index the 12 attention heads by $\{0,1,\ldots,11\}$. 
For a group size $g$, the number of groups is $K=12/g$. 
We use the following grouping rules.

\textbf{Adjacent grouping} partitions consecutive heads into the same group:
\[
    \{0,\ldots,g-1\},\ 
    \{g,\ldots,2g-1\},\ 
    \ldots,\ 
    \{12-g,\ldots,11\}.
\]
For example, when $g=3$, adjacent grouping gives
\[
    \{0,1,2\},\quad
    \{3,4,5\},\quad
    \{6,7,8\},\quad
    \{9,10,11\}.
\]

\textbf{Interval grouping} distributes nearby heads into different groups by using a fixed stride $K=12/g$. 
The $j$-th group is
\[
    \{j,\ j+K,\ j+2K,\ \ldots,\ j+(g-1)K\},
    \qquad j=0,\ldots,K-1.
\]
For example, when $g=3$, we have $K=4$, so interval grouping gives
\[
    \{0,4,8\},\quad
    \{1,5,9\},\quad
    \{2,6,10\},\quad
    \{3,7,11\}.
\]
When $g=6$, we have $K=2$, so interval grouping gives
\[
    \{0,2,4,6,8,10\},\quad
    \{1,3,5,7,9,11\}.
\]

\textbf{Random grouping} first samples a random permutation of the 12 heads and then splits the permuted list into contiguous groups of size $g$. 
For example, if $g=3$ and the sampled permutation is
\[
    (7,2,10,0,5,9,1,4,11,3,6,8),
\]
then the random groups are
\[
    \{7,2,10\},\quad
    \{0,5,9\},\quad
    \{1,4,11\},\quad
    \{3,6,8\}.
\]
In all random grouping experiments, the random partition is resampled at every optimizer step.

\paragraph{Fixed-QK+V and V-only Controls.}
The \textbf{Fixed-QK+V} experiments test whether additionally applying group-wise whitening to V improves upon the best QK-only configuration. 
In these runs, Q and K are fixed to random grouping with $g_{QK}=6$, while the group size and grouping rule of V are varied. 
The \textbf{V} experiments apply Group Muon only to the V projection, leaving Q and K on their full matrix Muon paths. 
These controls are used to separate the effect of grouping Q/K from the effect of grouping V.

\section{Experimental Observation}

\subsection{Nearly Full Rank}
\label{appendix:fullrank}

We empirically examine the rank structure of the attention(QK)-projection gradients and momentum during training. 
For each recorded step, we compute the rank of the full WQ/WK matrices and of their row-wise grouped submatrices. 
Ranks are normalized by the corresponding row dimension, so a value close to 1 indicates near-full row rank.

\begin{figure}[H]
\centering
\includegraphics[width=0.98\textwidth]{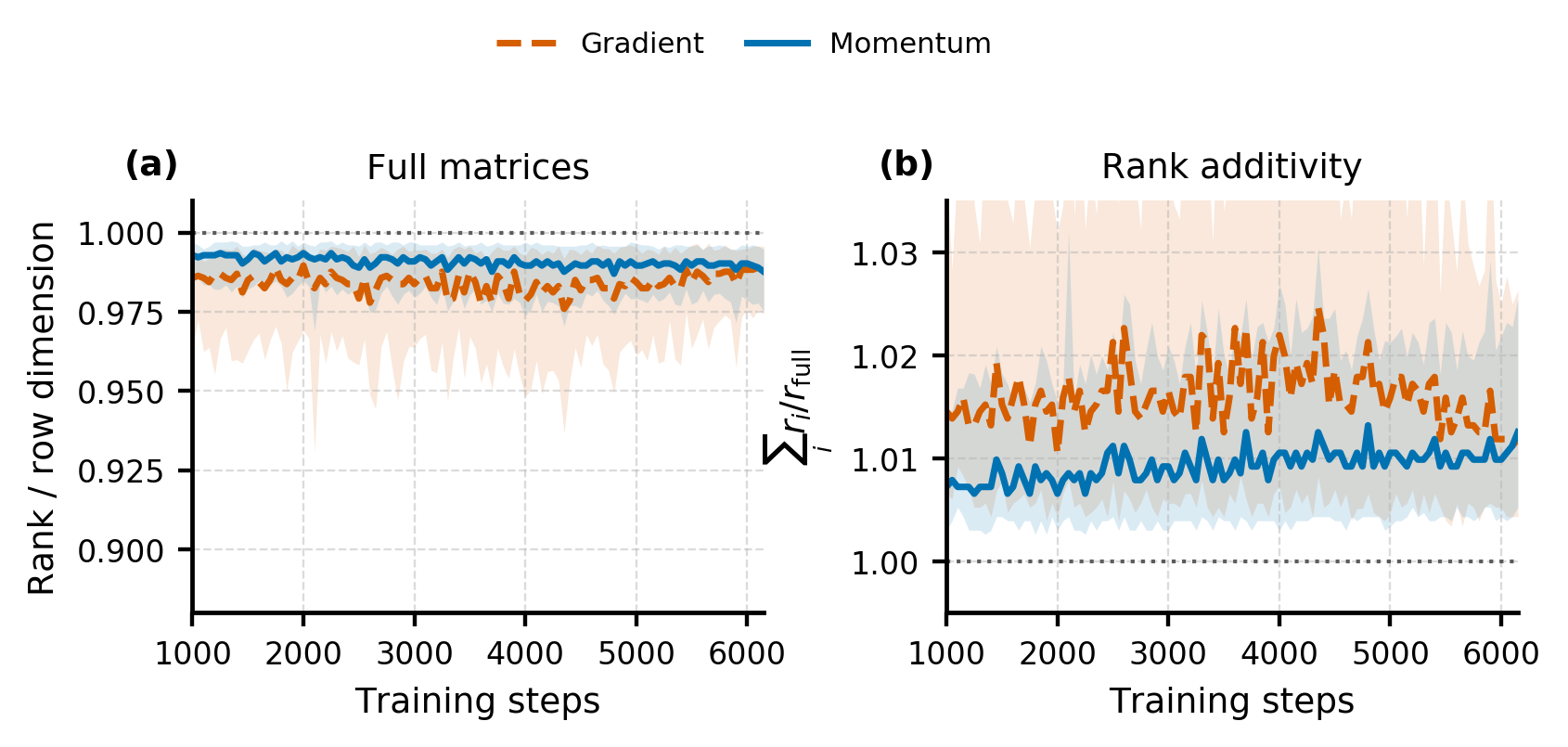}
\vspace{-0.75em}

\caption{
Rank behavior of gradients and momentum during training. 
Both full gradients and momentum remain close to full row rank, with momentum exhibiting a more stable rank ratio than gradients. 
For grouped updates, the sum of group ranks stays close to the full-matrix rank throughout the main training regime.
}
\label{fig:near_full_rank}
\vspace{-1em}
\end{figure}

Figure~\ref{fig:near_full_rank} shows that both the full matrices and their grouped row partitions remain close to full row rank after the early training phase. 
This supports the empirical condition used in our analysis: for a row-wise partition with group row sizes $m_i$, the full rank satisfies $r\approx m$, each group rank satisfies $r_i\approx m_i$, and the aggregate grouped rank $\sum_i r_i$ remains close to the full-matrix rank.

\subsection{Grouping-induced Norm Cost}
\label{appendix_Grouping-induced-Norm-Cost}

\paragraph{Non-ideal Muon whitening.}
The rank-dependent expression above relies on the ideal polar-factor identity
$
    \|P(G)\|_F^2=\mathrm{rank}(G).
$
For a practical Muon implementation, however, the whitening direction is produced by an approximate operator, such as a finite number of Newton--Schulz iterations. 
Let
$
    O_{\mathrm{full}}=\mathcal{A}(G)
$
denote the approximate whitened update computed from the full gradient, and let
$
    O_i=\mathcal{A}(G_i)
$
denote the approximate whitened update computed from the $i$-th group gradient. 
The group-wise update is then
\[
    O_{\mathrm{grp}}
    =
    \begin{bmatrix}
        O_1^\top&
        O_2^\top&
        \cdots&
        O_M^\top
    \end{bmatrix}^\top .
\]
Applying the smoothness bound to the full update gives
\[
    \mathcal{L}(W-\eta O_{\mathrm{full}})
    \le
    \mathcal{L}(W)
    -\eta\langle G,O_{\mathrm{full}}\rangle
    +\frac{\beta\eta^2}{2}\|O_{\mathrm{full}}\|_F^2 .
\]
Similarly, applying it to the group-wise update gives
\[
    \mathcal{L}(W-\eta O_{\mathrm{grp}})
    \le
    \mathcal{L}(W)
    -\eta\sum_{i=1}^M \langle G_i,O_i\rangle
    +\frac{\beta\eta^2}{2}\sum_{i=1}^M\|O_i\|_F^2 .
\]
Therefore, the additional second-order cost induced by replacing full-matrix whitening with group-wise whitening is
\[
    \frac{\beta\eta^2}{2}
    \left(
        \sum_{i=1}^M \|O_i\|_F^2
        -
        \|O_{\mathrm{full}}\|_F^2
    \right).
\]
Equivalently, after dividing the comparison criterion by $\eta>0$, the corresponding grouping-induced norm cost is
\[
    \frac{\beta\eta}{2}
    \left(
        \sum_{i=1}^M \|O_i\|_F^2
        -
        \|O_{\mathrm{full}}\|_F^2
    \right).
\]
Figure~\ref{fig3} shows that this Frobenius-norm gap is clearly nonzero in practice: head-wise whitening consistently yields larger $\sum_i \|O_i\|_F^2$ than the full-matrix counterpart $\|O_{\mathrm{full}}\|_F^2$.

\begin{figure}[H]
\centering
\includegraphics[width=0.98\textwidth]{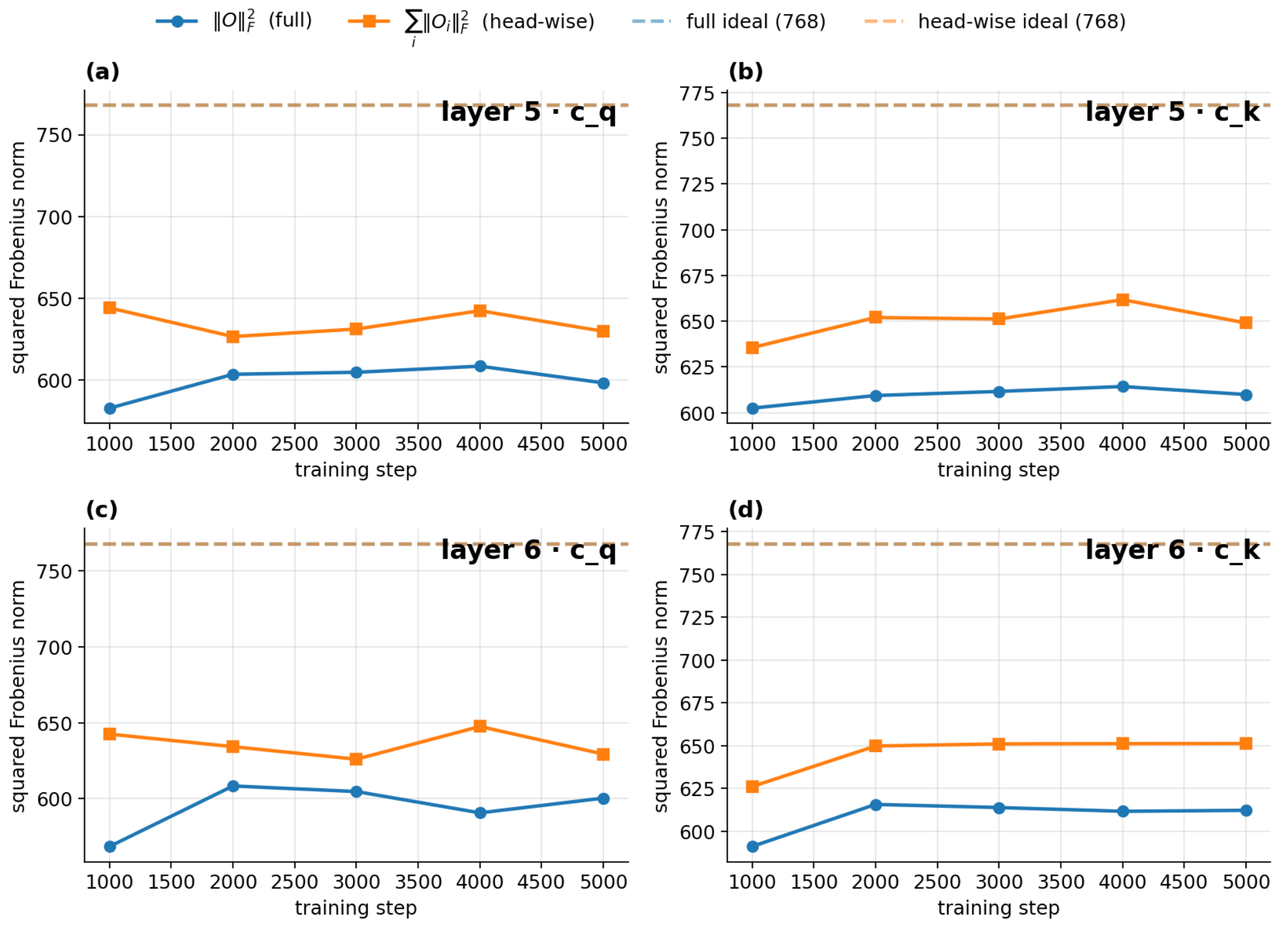}
\vspace{-0.75em}

\caption{
Squared Frobenius norm of the whitened Muon update under full-matrix whitening and head-wise whitening. 
Panels (a)--(d) report representative Q/K projections from layers 5 and 6; across all cases, head-wise whitening yields a larger $\sum_i \|O_i\|_F^2$ than the full-matrix counterpart $\|O\|_F^2$. 
This indicates that head-wise whitening introduces a larger grouping-induced norm cost, while both remain below the ideal full-rank value of 768.
}

\label{fig3}
\vspace{-1em}
\end{figure}

\section{Aligned Low-Rank Example}
\label{app:aligned-low-rank-example}

We provide a simple aligned low-rank example to illustrate that the gain--cost trade-off can already appear under ideal polar-factor Muon.
Consider $k$ nonzero group gradients that share the same right singular direction:
\[
    G_i=a_i u_i v^\top,
    \qquad
    \|u_i\|_2=\|v\|_2=1,
    \qquad
    i=1,\ldots,k ,
\]
where $a_i\neq 0$.
Then each nonzero group gradient has rank one.
At the same time, since all groups share the same right direction $v$, the concatenated gradient also has rank one.
Therefore,
\[
    \sum_{i=1}^k r_i-r = k-1 .
\]

The group-wise nuclear-norm gain is
\[
    \sum_{i=1}^k \|G_i\|_*-\|G\|_*
    =
    \sum_{i=1}^k |a_i|
    -
    \left(\sum_{i=1}^k a_i^2\right)^{1/2}.
\]
If the group strengths are equal, i.e., $|a_i|=a$, then
\[
    \sum_{i=1}^k \|G_i\|_*-\|G\|_*
    =
    a(k-\sqrt{k}).
\]
Thus the comparison criterion \eqref{ieq:main-criterion} becomes
\[
    a(k-\sqrt{k})
    >
    \frac{\beta\eta}{2}(k-1).
\]
This example shows that finer splitting can increase both the group-wise whitening gain and the grouping-induced norm cost.
Therefore, even under ideal polar-factor Muon, grouping is beneficial only when the former dominates the latter.

\end{document}